\documentclass[12pt,reqno]{amsart}

\usepackage{amsfonts}
\usepackage{amsmath}
\usepackage{amsthm}
\usepackage{mathrsfs}
\usepackage{amssymb}
\usepackage{stmaryrd}
\usepackage{graphicx}
\usepackage{epstopdf}
\usepackage{textcomp}
\newcommand{\Tr}{{\rm Tr\,}}
\newcommand{\conv}{{\rm conv\,}}
 \newtheorem{thm}{Theorem}[section]

 \theoremstyle{definition}
 
 \theoremstyle{remark}
 \newtheorem{rem}[thm]{Remark}
 
 \numberwithin{equation}{section}

\begin{document}

%
%
%
%
%
%
%
%
%

\title[Asymptotic sequential Rademacher complexity]
 {Asymptotic sequential Rademacher complexity\\ of a finite function class}

\author[D.B. Rokhlin]{Dmitry B. Rokhlin}

\address{%
Institute of Mathematics, Mechanics and Computer Sciences\\
              Southern Federal University\\
Mil'chakova str. 8a\\
344090, Rostov-on-Don\\
Russia
}

\email{rokhlin@math.rsu.ru}




\begin{abstract}
For a finite function class we describe the large sample limit of the sequential Rademacher complexity  in terms of the viscosity solution of a $G$-heat equation. In the language of Peng's sublinear expectation theory, the same quantity equals to the expected value of the largest order statistics of a multidimensional $G$-normal random variable. We illustrate this result by deriving upper and lower bounds for the asymptotic sequential Rademacher complexity.
\end{abstract}

\maketitle
\section{Preliminaries} \label{sec:1}
The notion of sequential Rademacher complexity was introduced in \cite{RakSriTew10} (see also \cite{RakSriTew15,RakSriTew15b}).  Let $(\varepsilon_i)_{i=1}^n$ be independent Rademacher random variables: $\mathsf P(\varepsilon_i=1)=\mathsf P(\varepsilon_i=-1)=1/2$.
Consider a set $\mathcal Z$, endowed with a $\sigma$-algebra $\mathscr G$, and a collection $\mathcal F$ of Borel measurable functions $f:\mathcal Z\mapsto\mathbb R$. For any sequence of functions $z_n:\{-1,1\}^{n-1}\mapsto\mathcal Z$, $n\ge 1$, where $z_1$ is simply an element of $\mathcal Z$, put
$$\mathfrak R_n(\mathcal F,z_1^n)=\frac{1}{\sqrt n}\mathsf E\sup_{f\in\mathcal F}\sum_{t=1}^n\varepsilon_t f(z_t(\varepsilon_1^{t-1})).$$
By $a_1^n$ we denote a sequence $(a_1,\dots,a_n)$. The \emph{sequential Rademacher complexity} of the function class $\mathcal F$ is defined by
\begin{equation} \label{1.1}
\mathfrak R_n(\mathcal F)=\sup_{z_1^n}\mathfrak R_n(\mathcal F,z_1^n).
\end{equation}

The incentives to study this quantity come from the online learning theory, where on every round $t$ a learner picks an element $q_t$ from the set $\mathcal Q$ of all probability distributions  defined on the Borel $\sigma$-algebra of the metric space $\mathcal F$, and an adversary picks an element $z_t\in\mathcal Z$. The value $\int_\mathcal F f(z_t)\,q_t(df)$ determines the loss of the learner. The normalized cumulative regret over $n$ rounds is defined by
$$ \mathscr R_n(q_1^n,z_1^n)=\frac{1}{\sqrt n}\left(\sum_{t=1}^n \int_\mathcal F f(z_t)\,q_t(df)-\inf_{f\in\mathcal F}\sum_{t=1}^n f(z_t)\right).$$
This quantity compares the regret of the randomized strategy $q_1^n$ with the regret of a best deterministic decision, taken in hindsight. Choosing their strategies, the learner and the adversary can use the information on all previous moves. Without going into the details, we only define the value of the repeated two-player game:
$$\mathcal V_n(\mathcal F)=\inf_{q_1\in\mathcal Q}\sup_{z_1\in\mathcal Z}\dots\inf_{q_n\in\mathcal Q}\sup_{z_n\in\mathcal Z}\mathscr R_n(q_1^n,z_1^n).$$

Typically, the sum $n^{-1/2}\sum_{t=1}^n \int_\mathcal F f(z_t)\,q_t(df)$ grows linearly in $\sqrt n$. The class $\mathcal F$ is called \emph{learnable} if:
$$ \limsup_{n\to\infty}\frac{\mathcal V_n(\mathcal F)}{\sqrt n}=0.$$
The following nice estimate was obtained in \cite[Theorem 2]{RakSriTew10}:
$\mathcal V_n(\mathcal F)\le 2\mathfrak R_n(\mathcal F).$
In the model of the supervised learning a similar lower bound also holds true: see \cite[Proposition 9]{RakSriTew15}.

In the sequel we assume that the class $\mathcal F$ is \emph{finite}: $\mathcal F=\{f_1,\dots,f_m\}$, and its elements are uniformly bounded: $|f_i|\le b$. Any such class is learnable: $\mathfrak R_n(\mathcal F)\le b\sqrt{2\ln m}$ (see \cite[Lemma 5]{RakSriTew10}, \cite[Lemma 1]{RakSriTew15b}). The goal of the present note is to characterize the quantity
$$\mathfrak R^a(\mathcal F)=\lim_{n\to\infty}\mathfrak R_n(\mathcal F),$$
which we call the \emph{asymptotic sequential Rademacher complexity} of $\mathcal F$.

The mentioned estimate of $\mathfrak R_n(\mathcal F)$ implies the inequality
\begin{equation} \label{1.2}
\mathfrak R^a(\mathcal F)\le b\sqrt{2\ln m}.
\end{equation}
Note, that \eqref{1.2} does not take into account the structure of the set $\mathcal F$. To get an insight into what features of $\mathcal F$ are essential, let us consider the \emph{Rademacher complexity}: a well established notion of a statistical learning theory, where the data $z_t$ are assumed to be independent and identically distributed. Let $(Z_t)_{t=1}^n$ be a sequence of i.i.d. random variables with values in $\mathcal Z$. It is assumed also that $(Z_t)_{t=1}^n$ are independent from $(\varepsilon_t)_{t=1}^n$. The Rademacher complexity of a function class $\mathcal F$ is defined by (see, e.g., \cite{SamWeb11,Vyu15})
\begin{equation} \label{1.3}
\mathfrak R_n^{iid}(\mathcal F)=\frac{1}{\sqrt n}\mathsf E\sup_{f\in\mathcal F}\sum_{t=1}^n\varepsilon_t f(Z_t).
\end{equation}
The role of is this quantity in the statistical learning theory is similar to the role of \eqref{1.1} in the online learning theory.

For the case of a finite class $\mathcal F=\{f_1,\dots,f_m\}$ one may rewrite \eqref{1.3} as
$$ \mathfrak R_n^{iid}(\mathcal F)=\mathsf Eg\left(\sum_{t=1}^n\frac{\varepsilon_t F(Z_t)}{\sqrt n}\right),\quad g(x)=\max\{x_1,\dots,x_m\},$$
where $F(z)=(f_1(z),\dots,f_m(z))$. Although $g$ is not bounded, the validity of the central limit theorem can be established with the use of the Hoeffding inequality as in \cite[Lemma A.11]{CesLug06} (see also the proof of Theorem \ref{th:2} below). Let $\Sigma$ be the covariance matrix of $\varepsilon F(Z)$, where $(\varepsilon,Z)$ is distributed as $(\varepsilon_t,Z_t)$. Then
\begin{equation} \label{1.4}
\mathfrak R^{a,iid}(\mathcal F):=\lim_{n\to\infty} \mathfrak R_n^{iid}(\mathcal F)=\mathsf E\max\{Y_1,\dots,Y_m\},\quad Y\sim N(0,\Sigma).
\end{equation}
Thus, the \emph{asymptotic Rademacher complexity} \eqref{1.4} coincides with the expected value of \emph{largest order statistics} of an $m$-dimensional normal random variable $Y$ with zero mean and the covariance matrix $$\Sigma_{kl}=(\mathsf E[f_k(Z)f_l(Z)])_{k,l=1}^m.$$

We will see that $\mathfrak R^a(\mathcal F)$ admits a representation similar to \eqref{1.4} in the framework of Peng's sublinear expectation theory \cite{Peng10}. The characterization of $\mathfrak R^a(\mathcal F)$ in terms of the viscosity solution of a $G$-heat equation is given in Theorem \ref{th:2}. This result is translated to the language of the sublinear expectation theory in Remark \ref{rem:2}. In Section \ref{sec:3} we obtain the upper bound \eqref{1.2}, as well as a lower bound for $\mathfrak R^a(\mathcal F)$, combining viscosity solutions techniques with known estimates of the expected maximum of a Gaussian process.

\section{The main result} \label{sec:2}
Our argumentation is based on a central limit theorem under model uncertainty (see \cite{Rok15}) which we now recall. Let $(\xi_i)_{i=1}^\infty$ be a sequence of $d$-dimensional random variables with zero mean and identity covariance matrix:
$$ \mathsf E\xi_i=0,\quad \mathsf E(\xi_i^k\xi_i^l)_{k,l=1}^d=I.$$
Let $\mathfrak A_1^n$ be the set of sequences $A_1^n=(A_i)_{i=1}^n$, where $A_i$ is a $\sigma(\xi_1,\dots,\xi_{i-1})$-measurable random element with values in a compact set $\Lambda$ of $d\times d$ matrices ($A_1$ is simply an element of $\Lambda$). For a bounded continuous function $f:\mathbb R^d\mapsto\mathbb R$ put
$$ \mathscr L=\lim_{n\to\infty}\sup_{A_1^n\in\mathfrak A_1^n}\mathsf Ef\left(\sum_{t=1}^n\frac{A_t\xi_t}{\sqrt n}\right).$$

Furthermore, let $G(S)=\frac{1}{2}\sup_{A\in\Lambda}\Tr(A A^T S)$, where $S$ belongs to the set $\mathbb S^m$ of symmetric $d\times d$ matrices. Consider the \emph{$G$-heat equation}
\begin{equation} \label{2.1}
-v_t(t,x)-G(v_{xx}(t,x))=0,\quad (t,x)\in Q^\circ=[0,1)\times\mathbb R^d,
\end{equation}
with the terminal condition
\begin{equation} \label{2.2}
v(1,x)=f(x).
\end{equation}
By $v_{xx}=(v_{x_i x_j})_{i,j=1}^n$ we denote the Hessian matrix.

Recall that an upper semicontinuous (usc) (resp., a lower semicontinuous (lsc)) function $u:Q\mapsto\mathbb R$, $Q=[0,1]\times\mathbb R^d$ is called a \emph{viscosity subsolution} (resp., \emph{supersolution}) of the problem \eqref{2.1}, \eqref{2.2} if
$$u(1,x)\le f(x),\quad (\text{resp.},\ u(1,x)\ge f(x)),$$
and for any $(\overline t,\overline x)\in Q^\circ=[0,1)\times\mathbb R^d$ and any test function $\varphi\in C^2(\mathbb R^{m+1})$ such that $(\overline t,\overline x)$ is a local maximum (resp., minimum) point of $u-\varphi$ on $Q^\circ$, the inequality
$$ (-\varphi_t-G(\varphi_{xx}))(\overline t,\overline x)\le 0\quad (\text{resp.},\ \ge 0)$$
holds true. A \emph{continuous} function $u:Q\mapsto\mathbb R$ is called a \emph{viscosity solution} of \eqref{2.1}, \eqref{2.2} if it is both viscosity sub- and supersolution. The classical reference is \cite{CraIshLio92}.

\begin{thm} \label{th:1}
Let $v:[0,1]\times\mathbb R^d\mapsto\mathbb R$ be the unique bounded viscosity solution of \eqref{2.1}, \eqref{2.2}. Then $\mathscr L=v(0,0)$.
\end{thm}

We refer to \cite{Rok15} for the proof of this result and the discussion of its relation to Peng's central limit theorem: \cite{Peng08}.

For $\mathcal F=\{f_1,\dots,f_m\}$ let us rewrite the expression \eqref{1.1} as follows:
\begin{equation} \label{2.3}
\mathfrak R_n(\mathcal F)=\sup_{z_1^n}\mathsf E g\left(\sum_{t=1}^n\frac{\varepsilon_t F(z_t(\varepsilon_1^{t-1}))}{\sqrt n}\right),\quad g(x)=\max\{x_1,\dots,x_m\},
\end{equation}
where $F=(f_1,\dots,f_m)$. Denote by $\Gamma$ the closure of the set $\{F(z):z\in\mathcal Z\}\subset\mathbb R^m.$
The expression \eqref{2.3} can be represented in the form
\begin{equation} \label{2.4}
\mathfrak R_n(\mathcal F)=\sup_{\gamma_1^n}\mathsf E g\left(\sum_{t=1}^n\frac{\varepsilon_t\gamma_t}{\sqrt n}\right),
\end{equation}
where supremum is taken over all sequences $\gamma_1^n$, whose elements $\gamma_t$ are measurable with respect to $\sigma(\varepsilon_1,\dots,\varepsilon_{t-1})$, and take values in $\Gamma$.

\begin{thm} \label{th:2}
Let $v:[0,1]\times\mathbb R^d\mapsto\mathbb R$ be the unique viscosity solution of the problem
\begin{equation} \label{2.5}
-v_t(t,x)-\frac{1}{2}\sup_{\gamma\in\Gamma}\sum_{i,j=1}^m \gamma^i\gamma^j v_{x_i x_j}=0,
\end{equation}
\begin{equation} \label{2.6}
v(1,x)=g(x)=\max\{x_1,\dots,x_m\}, \quad x\in\mathbb R^m,
\end{equation}
satisfying the linear growth condition: $|v(t,x)|\le C(1+|x|)$, where $|x|$ is the usual Euclidian norm of $x$. Then $\mathfrak R^a(\mathcal F)=v(0,0)$.
\end{thm}
\begin{proof} In Theorem \ref{th:1} it is not essential that matrices $A_t$ are quadratic.
So, to apply Theorem \ref{th:1} to the expression \eqref{2.4}, the only issue we need to overcome is the unboundedness of $g$.

The existence and uniqueness of a viscosity solution $v$ of \eqref{2.5}, \eqref{2.6}, satisfying the linear growth condition, is well known from the theory of stochastic optimal control: see Theorem 5.2 and Theorem 6.1 of \cite[Chapter 4]{YonZho99}.
Put $a\vee b=\max\{a,b\}$, $a\wedge b=\min\{a,b\}$, and denote by $v_L$ the unique bounded viscosity solution of \eqref{2.5} satisfying the terminal condition
$$v_L(x)=g_L(x):=g(x)\vee L\wedge(-L),\quad x\in\mathbb R^d$$
instead of \eqref{2.6}. We can apply Theorem \ref{th:1} to \eqref{2.4} with $g_L$ instead of $g$ and $\gamma_t\in\Gamma$ instead of quadratic matrices $A_t\in\Lambda$. As far as the equation \eqref{2.1} corresponds to \eqref{2.5}, we get
$$v_L(0,0)=\lim_{n\to\infty}\sup_{\gamma_1^n}\mathsf E g_L\left(\sum_{t=1}^n\frac{\varepsilon_t\gamma_t}{\sqrt n}\right).$$
So, it is sufficient to prove the relations
\begin{equation} \label{2.7}
\mathfrak R^a(\mathcal F)=\lim_{L\to\infty}\lim_{n\to\infty}\sup_{\gamma_1^n}\mathsf E g_L\left(\sum_{t=1}^n\frac{\varepsilon_t\gamma_t}{\sqrt n}\right),\quad v(0,0)=\lim_{L\to\infty} v_L(0,0).
\end{equation}

The proof of the first equality \eqref{2.7} is similar to that of \cite[Lemma A.11]{CesLug06}. Put
$X_n=n^{-1/2}\sum_{t=1}^n \varepsilon_t\gamma_t.$
From the identity
$$ g(x)=g_L(x)+(g(x)-L)I_{\{g(x)>L\}}+(g(x)+L)I_{\{g(x)<-L\}},$$
we get the inequalities
$$ \mathsf E g(X_n)\le \mathsf E g_L(X_n)+\mathsf E[(g(X_n)-L)I_{\{g(X_n)>L\}}], $$
$$ \mathsf E g(X_n)\ge \mathsf E g_L(X_n)+\mathsf E[(g(X_n)+L)I_{\{g(X_n)<-L\}}]. $$
Using the estimate $g(x)\le |x|$,  we obtain
\begin{align*}
 &\mathsf E[(g(X_n)-L)I_{\{g(X_n)>L\}}]\le \mathsf E[(|X_n|-L)I_{\{|X_n|>L\}}]=\mathsf E[(|X_n|-L)^+]\\
 &=\int_0^\infty\mathsf P((|X_n|-L)^+\ge u)du=\int_0^\infty\mathsf P(|X_n|\ge L+u)\,du\\
 &=\int_L^\infty\mathsf P(|X_n|\ge u)\,du\le\sum_{k=1}^m\int_L^\infty\mathsf P(|X_n^k|\ge u)du,\qquad a^+=\max\{a,0\}.
\end{align*}

Since $(\varepsilon_t\gamma_t^k)_{t=1}^n$ is a martingale difference and $|\varepsilon_t\gamma_t^k|\le b$, by the Azuma inequality (see, e.g., \cite[Theorem 1.3.1]{Ste97}):
$$ \mathsf P(\left(\left|\sum_{t=1}^n \varepsilon_t\gamma_t^k\right|\ge\lambda\right)\le 2\exp\left(-\frac{\lambda^2}{2 b^2n}\right)$$
 we get
$$ \mathsf P\left(\sqrt n |X^k_n|\ge \sqrt n u\right)\le 2\exp\left(-\frac{u^2}{2 b^2}\right).$$
It follows that
$$\mathsf E[(g(X_n)-L)I_{\{g(X_n)>L\}}]\le r(L):=2m\int_L^\infty \exp\left(-\frac{u^2}{2 b^2}\right)\,du.$$

Similarly,
$$\mathsf E[(g(X_n)+L)I_{\{g(X_n)<-L\}}]\ge -r(L).$$
Thus,
$$\mathsf E g_L(X_n)-r(L)\le \mathsf E g(X_n)\le \mathsf E g_L(X_n)+r(L),$$
and we get the inequalities
$$\sup_{\gamma_1^n} \mathsf E g_L(X_n)-r(L) \le \mathfrak R_n(\mathcal F)\le \sup_{\gamma_1^n}\mathsf E g_L(X_n)+r(L),$$
which imply the first equality \eqref{2.7}, since $r(L)\to 0$, $L\to\infty$,

Firthermore, put $G(X)=\frac{1}{2}\sup_{\gamma\in\Gamma}\sum_{i,j=1}^m X_{ij}\gamma^i\gamma^j$,
\begin{equation} \label{2.8}
 F(t,x,r,q,X)=\begin{cases}
-q-G(X),& t\in [0,1),\\
r-g(x),& t=1,
\end{cases}
\end{equation}
and denote by
$$F_*(t,x,r,q,X)=\begin{cases}
-q-G(X),& t\in [0,1),\\
\min\{-q-G(X),r-g(x)\},& t=1
\end{cases}$$
the lsc envelope of $F:[0,1]\times\mathbb R^m\times\mathbb R\times\mathbb R\times\mathbb S^m\mapsto\mathbb R$. A usc function $u$ is a viscosity solution of \eqref{2.5}, \eqref{2.6} if and only if
\begin{equation} \label{2.9}
F_*(\overline t,\overline x,u(\overline t,\overline x),\varphi_t(\overline t,\overline x),\varphi_{xx}(\overline t,\overline x))\le 0
\end{equation}
for any $(\overline t,\overline x)\in Q$ and any test function $\varphi\in C^2(\mathbb R^{m+1})$ such that
$(\overline t, \overline x)$ is a local maximum point of $u-\varphi$ on $Q$. To prove this we only need to show that if $u$ is a viscosity subsolution in the sense of the definition \eqref{2.9}, then the inequality
$$(-\varphi_t-G(\varphi_{xx}))(1,\overline x)\le 0$$
is impossible. Note, that $\widehat \varphi=\varphi+c(1-t)$ is still a test function for $u$ at $(1,\overline x)$ for any $c>0$. Thus,
$$ c-\varphi_t(1,\overline x)-G(\varphi_{xx}(1,\overline x))\le 0,$$
and we get a contradiction since $c$ is arbitrary.

An advantage of the definition \eqref{2.9} is that it treats the equation and boundary condition simultaneously. As we have just seen, the correspondent boundary condition in the \emph{viscosity sense}, given by \eqref{2.9} for $t=1$ (cf. \cite[\S 7]{CraIshLio92}), is equivalent to the usual boundary condition in our case.

Viscosity supersolutions are considered in the same way. A viscosity solution $v$ of \eqref{2.5}, \eqref{2.6} may be termed as a viscosity solution of the equation
\begin{equation} \label{2.10}
 F(t,x,v(t,x),v_t(t,x),v_{xx}(t,x))=0,\quad (t,x)\in Q.
\end{equation}

Denote $v_+$ (resp., $v_-$) the viscosity solution of \eqref{2.5}, satisfying the terminal condition $v_+(1,x)=g^+(x)$ (resp., $v_-(1,x)=g^-(x)$), $x\in\mathbb R^m$, and the linear growth condition. By the comparison result of \cite[Theorem 2.1]{DaLLey06} or \cite[Theorem 5]{Str08} it follows that
$$ v_-\le v_L\le v_+\quad \textrm{on } [0,1]\times\mathbb R^d.$$
Hence, the upper and lower ``relaxed limits'' (see \cite[\S 6]{CraIshLio92}, \cite[Chapter 2]{Gig06})
$$\overline v(t,x) =\lim_{j\to\infty}\sup\{v_L(s,y):L\ge j,\ (s,y)\in Q,\ |s-t|+|y-x|\le 1/j\},$$
$$\underline v(t,x)=\lim_{j\to\infty}\inf\{v_L(s,y):L\ge j,\ (s,y)\in Q,\ |s-t|+|y-x|\le 1/j\} $$
are finite and satisfy the linear growth condition. Moreover, $\overline v$ is usc, $\underline v$ is lsc.

Denote by $F_L$ the function of the form \eqref{2.8}, where $g$ is changed to $g_L$. The lower relaxed limit of the lsc envelope $(F_L)_*$ of $F_L$ is $F_*$.
By \cite[Theorem 2.3.5]{Gig06} it follows that the function $\overline v$ is a usc subsolution of \eqref{2.10}. Similary, $\underline v$ is an lsc supersolution of \eqref{2.10}
By the mentioned comparison results of \cite{DaLLey06} or \cite{Str08} we have $\overline v\le \underline v$. The opposite inequality is clear from the definition of $\overline v$, $\underline v$. It follows that the function $v=\overline v=\underline v$ coincides with the unique viscosity solution of \eqref{2.5}, \eqref{2.6}, and the second equality \eqref{2.7} holds true:
\[ \lim_{L\to\infty} v_L(0,0)=v(0,0).\qedhere \]
\end{proof}

\begin{rem} \label{rem:1}
As already mentioned, there is link between the problem (\ref{2.5}), (\ref{2.6}) and the stochastic control theory. Let $(W_t)_{t\ge 0}$ be a Brownian motion. Denote by  $\mathfrak G$ the set of stochastic processes $\gamma$ adapted to the natural filtration of $(W_t)_{t\ge 0}$ and taking values in $\Gamma$.
Consider the family of stochastic processes
$$ X_s^{t,x,\gamma,i}=x+\int_t^s\gamma^i_u\,dW_u,\quad s\in[t,1],\quad i=1,\dots,m$$
and the related value function
$$ v(t,x)=\sup\left\{\mathsf E\max\{X_1^{t,x,\gamma,1},\dots,X_1^{t,x,\gamma,m}\}:\gamma\in\mathfrak G\right\}.$$
By Proposition 3.1 and Theorem 5.2 of \cite[Chapter 4]{YonZho99}, $v$ is a viscosity solution of \eqref{2.5}, \eqref{2.6}, satisfying the linear growth condition. In particular,
$$ \mathfrak R^a(\mathcal F)=v(0,0)=\sup\left\{\mathsf E\max\left\{\int_0^1\gamma^1_u\,dW_u,\dots,\int_0^1\gamma^m_u\,dW_u\right\}:\gamma\in\mathfrak G\right\}.$$
\end{rem}

\begin{rem} \label{rem:2}
Denote by $\conv A$ the convex hull of a set $A$. Let us rewrite the equation \eqref{2.5} in the form
$$ -v_t(t,x)-\frac{1}{2}\sup_{Q\in\Theta}\Tr(Qv_{xx}(t,x))=0,$$
where $\Theta=\conv\{(\gamma^i\gamma^j)_{i,j=1}^m:\gamma\in\Gamma\}\subset\mathbb S^n.$ In the framework of the sublinear expectation theory we have (see \cite[Chapter II]{Peng10})
\begin{equation} \label{2.11}
 \mathfrak R^a(\mathcal F)=v(0,0)=\widehat{\mathsf E}\max\{X_1,\dots,X_m\},
\end{equation}
where $X$ is a multidimensional $G$-normal random variable: $X\sim\mathcal N(0,\Theta)$, and by $\widehat{\mathsf E}$ we denote a sublinear expectation. Thus, $\mathfrak R^a(\mathcal F)$ can be regarded as the \emph{sublinear expected value of the largest order statistics} of a  multidimensional \emph{$G$-normal} random variable. Note, that the set $\Theta$, characterizing the uncertainty structure of $Y$, coincides with the convex hull of covariance matrices of random vectors $\varepsilon\gamma$, $\gamma\in\Gamma=\{F(z):z\in\mathcal Z\}$, where $\varepsilon$ is a Rademacher random variable. We emphasize the similarity of this description with case of the Rademacher complexity $\mathfrak R^{a,iid}(\mathcal F)$, considered in Section \ref{sec:1}.
\end{rem}

\section{Upper and lower bounds} \label{sec:3}
To illustrate our approach, we derive upper and lower bounds for $\mathfrak R^a(\mathcal F)$, combining simple comparison results for viscosity solutions of parabolic equations and known estimates of the expected maximum of a Gaussian process.
\begin{thm} \label{th:3}
Let $\mathcal F=\{f_1,\dots,f_m\}$, where $f_i$ are uniformly bounded $|f_i|\le b$. Then
\begin{equation} \label{3.1}
\frac{1}{17} a(\mathcal F)\le\frac{\mathfrak R^a(\mathcal F)}{\sqrt{\ln m}}\le \sqrt{2} b,
\end{equation}
$$a(\mathcal F)=\sup_{\nu\in\mathcal P(\mathscr G)}\inf_{i\neq j}\left(\int_{\mathcal Z}(f_i(z)-f_j(z))^2\,\nu(dz)\right)^{1/2},$$
where $\mathcal P(\mathscr G)$ is the set of probability measures on the $\sigma$-algebra of $\mathscr G$.
\end{thm}
\begin{proof}
Along with \eqref{2.5} consider the usual heat equation
\begin{equation} \label{3.2}
-u_t(t,x)-\frac{b^2}{2}\Tr(u_{xx})=0, \quad (t,x)\in Q^\circ
\end{equation}
with the terminal condition $u(1,x)=g(x)$. The function $U=e^{1-t} u$ satisfies the equation
\begin{equation} \label{3.3}
-U_t+U-\frac{b^2}{2}\Tr(U_{xx})=0
\end{equation}
and the same terminal condition. Similarly, if $v$ is the viscosity solution of \eqref{2.5}, \eqref{2.6}, then the function $V=e^{1-t} v$ satisfies the equation
\begin{equation} \label{3.4}
-V_t+V-\frac{1}{2}\sup_{\gamma\in\Gamma}\sum_{i,j=1}^m \gamma^i\gamma^j V_{x_i x_j}=0,
\end{equation}
in $Q^\circ$ in the viscosity sense, and $V(1,x)=g(x)$.

Assume that there exists a point $(\overline t,\overline x)\in Q$ such that $(V-U)(\overline t,\overline x)>0.$ In view of the terminal conditions, we have $\overline t<1$. Since $U$, $V$ satisfy the linear growth condition, the function
$$ (V-U)(t,x)-\frac{\varepsilon}{2}|x|^2$$
attains its maximum on $Q$ at some point $(t_\varepsilon,x_\varepsilon)$. For $\varepsilon$ small enough one may assume that $t_\varepsilon<1$ by virtue of the inequality
\begin{equation}  \label{3.5}
\sup_{(t,x)\in Q}\left((V-U)(t,x)-\frac{\varepsilon}{2}|x|^2\right)\ge (V-U)(\overline t,\overline x)-\frac{\varepsilon}{2}|\overline x|^2>0.
\end{equation}

By the definition, $U+\varepsilon |x|^2/2$ is a test function for the viscosity subsolution $V$ of \eqref{3.4} at $(t_\varepsilon,x_\varepsilon)$. Hence,
\begin{equation} \label{3.6}
\left(-U_t+ V-\frac{1}{2}\sup_{\gamma\in\Gamma}\langle (U_{xx}+\varepsilon I)\gamma,\gamma\rangle\right)(t_\varepsilon,x_\varepsilon)\le 0,
\end{equation}
where $\langle\cdot,\cdot\rangle$ is the usual scalar product in $\mathbb R^m$.
From an explicit representation of $u$:
$$ u(t,x)=\frac{1}{(b\sqrt{2\pi(1-t)})^m}\int_{\mathbb R^m} \exp\left(-\frac{|y|^2}{2 b^2(1-t)}\right)g(x+y)\,dy$$
and the convexity of $g$ it follows that $U$ is convex in $x$. Thus, $U_{xx}$ is non-negative definite and
\begin{equation} \label{3.7}
\sup_{\gamma\in\Gamma}\langle U_{xx}(t_\varepsilon,x_\varepsilon)\gamma,\gamma\rangle\le \sup_{|\gamma|\le b}\langle U_{xx}(t_\varepsilon,x_\varepsilon)\gamma,\gamma\rangle\le b^2(\Tr U_{xx})(t_\varepsilon,x_\varepsilon).
\end{equation}
From the inequalities \eqref{3.6}, \eqref{3.7} and the equation \eqref{3.3}, we get
$$ V(t_\varepsilon,x_\varepsilon)\le \left(U_t+\frac{b^2}{2}(\Tr U_{xx})\right)(t_\varepsilon,x_\varepsilon)+ \frac{b^2}{2}\varepsilon=U(t_\varepsilon,x_\varepsilon)+ \frac{b^2}{2}\varepsilon.$$
Combining this with \eqref{3.5}:
$$ 0<(V-U)(\overline t,\overline x)\le \frac{\varepsilon}{2}|\overline x|^2 +(V-U)(t_\varepsilon,x_\varepsilon)\le \frac{\varepsilon}{2}|\overline x|^2 + \frac{b^2}{2}\varepsilon,$$
we get a contradiction by letting $\varepsilon\to 0$.

Thus, $V\le U$. In particular, $\mathfrak R^a(\mathcal F)=v(0,0)\le u(0,0)$. To get the right inequality \eqref{3.1} we use the probabilistic representation of $u$:
$$ \mathfrak R^a(\mathcal F)\le u(0,0)=\mathsf E g(b W_T)=b\mathsf E\max\{W_1^1,\dots,W_1^m\}\le b\sqrt{2\ln m},$$
where $W_1^i$ are independent standard normal random variables. The last inequality is taken from \cite{CesLug06} (Lemma A.13).

To obtain the left inequality \eqref{3.1} compare the representations \eqref{1.4} and \eqref{2.11}. Since $\Sigma\subset\Theta$, we conclude that $\mathfrak R^{a,iid}(\mathcal F)\le\mathfrak R^a(\mathcal F)$. As in the first part of the proof, this is a consequence of a comparison result: see \cite{Peng08b}. Applying to \eqref{1.4} the Sudakov inequality (see \cite[Lemma 5.5.6]{MarRos06}, \cite[Lemma 2.1.2]{Tal05}), and taking into account that $Z$ is arbitrary, we get
$$\mathfrak R^a(\mathcal F)\ge \frac{1}{17} a(\mathcal F) \sqrt{\ln m},$$
\begin{align*}
 a(\mathcal F)&=\sup_Z\inf_{i\neq j}\left(\mathsf E(Y_i-Y_j)^2\right)^{1/2}=\sup_Z\inf_{i\neq j}\left(\mathsf E(f_i(Z)-f_j(Z))^2\right)^{1/2}=
 \\&=\sup_{\nu\in\mathcal P(\mathscr G)}\inf_{i\neq j}\left(\int_{\mathcal Z}(f_i(z)-f_j(z))^2\,\nu(dz)\right)^{1/2}.\qedhere
\end{align*}
\end{proof}
Assuming that $a(\mathcal F)\ge c>0$ uniformly in $m$, from \eqref{3.1} we see that $\mathcal R(\mathcal F)\sim\sqrt{\ln m}$ for large cardinality of $\mathcal F$. The factor $1/17$ in the lower bound \eqref{3.1} can be refined: see \cite[Section 2.3]{Fer75}.

It would be interesting to extend the representation \eqref{2.11} to the case of an infinite function class $\mathcal F$.

\subsection*{Acknowledgment}
The research is supported by Southern Federal University, project 213.01-07-2014/07.

\end{document}